\documentclass{article}

\usepackage{microtype}
\usepackage{graphicx}
\usepackage{subfigure}
\usepackage{booktabs} 

\usepackage{hyperref}

\usepackage[accepted]{icml2019_kwang}

\usepackage{hyperref}
\usepackage{url}

\usepackage{dsfont} 
\usepackage{xspace, amsmath, amssymb, amsthm, bbm, bm}
\usepackage{xcolor}

\usepackage{graphicx}
\usepackage{tabularx}

\usepackage{algorithm}
\usepackage{algorithmic}

\usepackage{multirow}
\usepackage{hhline}

\usepackage{booktabs}

\usepackage{capt-of}

\newtheoremstyle{kjstyle}
  {.3em} 
  {\topsep} 
  {\itshape} 
  {} 
  {\bfseries} 
  {.} 
  {.5em} 
  {} 

\theoremstyle{kjstyle}\newtheorem{lem}{Lemma}
\theoremstyle{kjstyle}\newtheorem{thm}{Theorem}
\theoremstyle{kjstyle}\newtheorem{cor}{Corollary}

\def\hS{\ensuremath{\hat{S}}\xspace}

\allowdisplaybreaks 

\def\T{\ensuremath{\top}}  
\def\sig{\ensuremath{\sigma}\xspace}

\def\diag{\ensuremath{\mbox{diag}}\xspace}

\def\dt{{\ensuremath{\delta}\xspace} }

\def\bfSigma{\ensuremath{\mathbf{\Sigma}}\xspace}

\newcommand{\what}[1]{ {\ensuremath{\widehat{#1}}} }

\def\sig{\ensuremath{\sigma}\xspace}

\usepackage[export]{adjustbox}

\usepackage{pbox} 

\newcommand{\fr}[2]{ { \frac{#1}{#2} }}
\def\lt{\left}
\def\rt{\right}
\def\gam{{\ensuremath{\gamma}\xspace} }

\makeatletter
\newcommand{\vast}{\bBigg@{3}}
\newcommand{\Vast}{\bBigg@{4}}
\makeatother

\def\cX{\ensuremath{\mathcal{X}}\xspace} 

\def\x{{{\mathbf x}}}
\def\y{{{\mathbf y}}}
\def\z{{{\mathbf z}}}
\def\la{{\langle}}
\def\ra{{\rangle}}

\def\lam{\ensuremath{\lambda}}

\def\U{\ensuremath{\mathbf{U}}\xspace} 
\def\S{\ensuremath{\mathbf{S}}\xspace} 
\def\V{\ensuremath{\mathbf{V}}\xspace}

\usepackage{pifont}
\def\u{{{\mathbf u}}}

\def\cA{\ensuremath{\mathcal{A}}\xspace}

\newcommand*{\mathcolor}{}
\def\mathcolor#1#{\mathcoloraux{#1}}
\newcommand*{\mathcoloraux}[3]{%
  \protect\leavevmode
  \begingroup
  \color#1{#2}#3%
  \endgroup
}

\usepackage{enumitem}

\def\Regret{\ensuremath{\text{Regret}}}

\newcommand{\wtil}[1]{ {\ensuremath{\widetilde{#1}}} }

\def\X{\ensuremath{\mathbf{X}}} 
\def\I{\ensuremath{\mathbf{I}}} 
\def\til{\tilde} 
\def\tvec{\text{vec}} 
\def\A{\ensuremath{\mathbf{A}}} 
\def\M{\ensuremath{\mathbf{M}}}

\def\cZ{\ensuremath{{\mathcal Z}}\xspace}

\def\hth{\ensuremath{\what{\theta}}\xspace}

\def\hth{{\what{\boldsymbol \theta}}} 

\def\X{\ensuremath{\mathbf{X}}\xspace}
\def\U{\ensuremath{\mathbf{U}}\xspace}
\def\x{\ensuremath{\mathbf{x}}\xspace}

\def\lam{{\ensuremath{\lambda}\xspace} }

\def\bfTh{{{\boldsymbol \Theta}}}

\def\v{\ensuremath{\mathbf{v}}}
\def\a{\ensuremath{\mathbf{a}}}

\def\vec{\ensuremath{\text{vec}}}
\def\Th{{{\boldsymbol \Theta}}} 
\def\K{\ensuremath{\mathbf{K}}\xspace} 
\def\Z{\ensuremath{\mathbf{Z}}} 
\def\Lam{{{\boldsymbol{{\Lambda}}}}}
\def\bflam{{{\boldsymbol{{\lambda}}}}}

\def\hU{\ensuremath{\what{\U}}\xspace} 
\def\hV{\ensuremath{\what{\V}}\xspace} 
\def\hS{\ensuremath{\what{\S}}\xspace} 
\def\hTh{\ensuremath{\what{\Th}}\xspace}

\def\R{\ensuremath{\mathbf{R}}} 
 
\def\E{\ensuremath{\mathbf{E}}} 
\def\hUp{\ensuremath{\what{\U}_\perp}}

\def\hK{\ensuremath{\what{\K}}}

\def\cO{{\ensuremath{\mathcal{O}}}} 
\def\tilO{{\ensuremath{\widetilde{\mathcal{O}}}}}

\def\Lam{{{\boldsymbol{{\Lambda}}}}}

\def\tr{{\ensuremath{\normalfont{\text{tr}}}}}

\def\th{{\ensuremath{\boldsymbol{\theta}}}}

\def\lamp{\ensuremath{\lambda_\perp}}
\usepackage{mdframed}
\usepackage{lipsum}
\definecolor{kjgray}{rgb}{.7,.7,.7}

\makeatletter
\renewcommand{\paragraph}{%
  \@startsection{paragraph}{4}%
  {\z@}{0.50ex \@plus 1ex \@minus .2ex}{-1em}%
  {\normalfont\normalsize\bfseries}%
}
\makeatother
\def\hVp{\ensuremath{\what{\V}_\perp}}

\def\hbS{\ensuremath{\what{\S}}}

\DeclareMathOperator{\EE}{\mathds{E}}

\def\RR{\ensuremath{\mathds{R}}}

\usepackage{todonotes}

\def\tilK{\ensuremath{\widetilde{\K}}}

\def\ddefloop#1{\ifx\ddefloop#1\else\ddef{#1}\expandafter\ddefloop\fi}
\def\ddef#1{\expandafter\def\csname c#1\endcsname{\ensuremath{\mathcal{#1}}}}
\ddefloop ABCDEFGHIJKLMNOPQRSTUVWXYZ\ddefloop
\def\ddef#1{\expandafter\def\csname #1#1\endcsname{\ensuremath{\mathbb{#1}}}}
\ddefloop ABCDEFGHIJKLMNOPQRSTUVWXYZ\ddefloop

\let\EE\undefined

\DeclareMathOperator{\EE}{\mathbb{E}}

\RequirePackage[OT1]{fontenc}
\usepackage[export]{adjustbox}

\renewcommand{\cite}{\citep}

\usepackage{booktabs}       
\usepackage{amsfonts}       
\usepackage{nicefrac}       
\usepackage{microtype}      

\usepackage{enumitem}
\setlist[itemize]{topsep=1pt,itemsep=0pt,parsep=2pt}

\usepackage{amsthm}
\usepackage{mathtools}
\usepackage{amsmath}
\usepackage{bbm}
\usepackage{amsfonts}
\usepackage{amssymb}
\let\vec\undefined
\usepackage{MnSymbol} 

\usepackage{etoolbox}
\newcommand{\zerodisplayskips}{%
  \setlength{\abovedisplayskip}{3pt}%
  \setlength{\belowdisplayskip}{4pt}%
  \setlength{\abovedisplayshortskip}{3pt}%
  \setlength{\belowdisplayshortskip}{3pt}}
\appto{\normalsize}{\zerodisplayskips}
\appto{\small}{\zerodisplayskips}
\appto{\footnotesize}{\zerodisplayskips}

\def\v{\ensuremath{\mathbf{v}}}
\def\u{\ensuremath{\mathbf{u}}}
\def\vec{\ensuremath{\text{vec}}}

\usepackage{anyfontsize}

\usepackage{titlesec}
\titlespacing*{\section}
{0pt}{.8ex}{.4ex}
\titlespacing*{\subsection}
{0pt}{.6ex}{.3ex}

\icmltitlerunning{Bilinear Bandits with Low-rank Structure}

\begin{document}

\twocolumn[
\icmltitle{Bilinear Bandits with Low-rank Structure}



\icmlsetsymbol{equal}{*}

\begin{icmlauthorlist}
\icmlauthor{Kwang-Sung Jun}{to}
\icmlauthor{Rebecca Willett}{goo}
\icmlauthor{Stephen Wright}{ed}
\icmlauthor{Robert Nowak}{ed}
\end{icmlauthorlist}

\icmlaffiliation{to}{Boston University}
\icmlaffiliation{goo}{University of Chicago}
\icmlaffiliation{ed}{University of Wisconsin-Madison}

\icmlcorrespondingauthor{Kwang-Sung Jun}{kwangsungjun@gmail.com}

\icmlkeywords{Multi-armed bandits, Linear bandits, Stochastic bandits, Low-rank structure}

\vskip 0.3in
]



\printAffiliationsAndNotice{}  

\begin{abstract}
  We introduce the bilinear bandit problem with low-rank structure in which an action takes the form of a pair of arms from two different entity types, and the reward is a bilinear function of the known feature vectors of the arms.
  The unknown in the problem is a $d_1$ by $d_2$ matrix $\mathbf{\Theta}^*$ that defines the reward, and has low rank $r \ll \min\{d_1,d_2\}$.
  Determination of $\mathbf{\Theta}^*$ with this low-rank structure poses a significant challenge in finding the right exploration-exploitation tradeoff.
  In this work, we propose a new two-stage algorithm called ``Explore-Subspace-Then-Refine'' (ESTR). The first stage is an explicit subspace exploration, while the second stage is a linear bandit algorithm called ``almost-low-dimensional OFUL'' (LowOFUL) that exploits and further refines the estimated subspace via a regularization technique.
  We show that the regret of ESTR is $\widetilde{\mathcal{O}}((d_1+d_2)^{3/2} \sqrt{r T})$ where $\widetilde{\mathcal{O}}$ hides logarithmic factors and $T$ is the time horizon, which improves upon the regret of $\widetilde{\mathcal{O}}(d_1d_2\sqrt{T})$ attained for a na\"ive linear bandit reduction.
  We conjecture that the regret bound of ESTR is unimprovable up to polylogarithmic factors, and our preliminary experiment shows that ESTR outperforms a na\"ive linear bandit reduction.
  \vspace{-1.5em}
\end{abstract}

\section{Introduction}


Consider a drug discovery application where scientists would like to choose a (drug, protein) pair and measure whether the pair exhibits the desired interaction~\cite{luo17anetwork}.
Over many repetitions of this step, one would like to maximize the number of discovered pairs with the desired interaction. 
Similarly, an online dating service may want to choose a (female, male) pair from the user pool, match them, and receive feedback about whether they like each other or not.
For clothing websites, the recommendation system may want to choose a pair of items (top, bottom) for a customer, whose appeal depends in part on whether they match.
In these applications, the two types of entities are recommended and evaluated as a unit.  
Having feature vectors of the entities available,\footnote{
  The feature vectors can be obtained either directly from the entity description (for example, hobbies or age) or by other preprocessing techniques (for example, embedding).
} the system must explore and learn what features of the two entities \emph{jointly} predict positive feedback
in order to make effective recommendations.


The recommendation system aims to obtain large rewards (the amount of
positive feedback) but does not know ahead of time the relationship
between the features and the feedback.  The system thus faces two
conflicting goals: choosing pairs that $(i)$ maximally help estimate
the relationship (``exploration'') but which may give small rewards
and $(ii)$ return relatively large, but possibly suboptimal, rewards
(``exploitation''), given the limited information obtained from the
feedback collected so far.  Such an exploration-exploitation dilemma
can be formulated as a multi-armed bandit
problem~\cite{robbins85asymptotically,auer02finite}.  When the feature
vectors are available for each arm, one can postulate simple reward
structures such as (generalized) linear models to allow a large or
even infinite number of arms~\cite{auer02using, dani08stochastic,
  ay11improved, filippi10parametric}, a paradigm that has received
much attention during the past decade, with such applications as
online news recommendations~\cite{li10acontextual}. Less is known for
the situation we consider here, in which the recommendation (action)
involves two different entity types and forms a bilinear structure.
The closest work we are aware of is \citet{kveton17stochastic_arxiv}
whose action structure is the same as ours but without arm feature vectors.
Factored bandits~\cite{zimmert18factored} provide a more general view with $L$ entity types rather than two, but they do not utilize arm features nor the low-rank structure.
Our problem is different from dueling bandits~\cite{yue12thekarmed} or
bandits with unknown user segment~\cite{bhargava17active}, which
choose two arms from \textit{the same} entity set rather than from two
\textit{different} entity types.  Section~\ref{sec:related} below
contains detailed comparisons to related work.

This paper introduces the bilinear bandit problem with low-rank
structure.  In each round $t$, an algorithm chooses a left arm $\x_t$
from $\cX \subseteq \RR^{d_1}$ and a right arm $\z_t$ from
$\cZ\subseteq \RR^{d_2}$, and observes a noisy reward of a bilinear
form:
\begin{align} \label{eq:noise-model}
  y_t = \x_t^\T \bfTh^*\z_t + \eta_t \;,
\end{align}
where $\bfTh^*\in\RR^{d_1 \times d_2}$ is an unknown parameter and
$\eta_t$ is a $\sig$-sub-Gaussian random variable conditioning on $\x_t$, $\z_t$, and all
the observations before (and excluding) time $t$.
Denoting by $r$ the rank of
$\bfTh^*$, we assume that $r$ is small ($r \ll \min\{d_1,d_2\}$),
which means that the reward is governed by a few factors. 
Such low-rank appears in many recommendation applications~\cite{ma08sorec}.
Our choice of reward model is popular and arguably natural; for example, the same model was used in \citet{luo17anetwork} for drug discovery.

The goal is to maximize the cumulative reward up to time $T$.
Equivalently, we aim to minimize the cumulative regret:\footnote{This
  regret definition is actually called \textit{pseudo} regret; we
  refer to~\citet[Section~1]{bubeck12regret} for detail.}
\begin{align}
\Regret_T = \sum_{t=1}^T \left\{ \max_{\x\in\cX, \z\in\cZ} \x^\T \bfTh^* \z - \x_t^\T \bfTh^* \z_t \right\}~.
\end{align}
A naive approach to this problem is to reduce the bilinear problem to
a linear problem, as follows:
\begin{align}\label{eq:oful-reduction}
\x^\T \bfTh^* \z = \la \tvec(\x\z^\T), \tvec(\bfTh^*) \ra  \;.
\end{align}
Throughout the paper, we focus on the regime in which the numbers of
possible actions $N_1 := |\cX| \in \NN_+ \cup \{\infty\}$ and $N_2 :=
|\cZ| \in \NN_+ \cup \{\infty\}$ are much larger than dimensions $d_1$
and $d_2$, respectively.\footnote{Otherwise, one can reduce the
  problem to the standard $K$-armed bandit problem and enjoy regret of
  $\tilO(\sqrt{N_1 N_2 T})$. With SupLinRel~\cite{auer02using}, one
  may also achieve $\tilO(\sqrt{d_1d_2 T \log (N_1N_2)})$, but this
  approach wastes a lot of samples and does not allow an infinite
  number of arms.}  The reduction above allows us to use the standard
linear bandit algorithms (see, for example,~\cite{ay11improved}) in
the $d_1 d_2$-dimensional space and achieve regret of $\tilO(d_1
d_2\sqrt{T})$, where $\tilO$ hides logarithmic factors. However,
$d_1d_2$ can be large, making this regret bound take an undesirably
large value. Moreover, the regret does not decrease as $r$ gets
smaller, since the reduction hinders us from exploiting the low-rank
structure.

We address the following challenge: Can we design an algorithm for the
bilinear bandit problem that exploits the low-rank structure and
enjoys regret strictly smaller  than $\tilO(d_1 d_2\sqrt{T})$?  We answer the
question in the affirmative by proposing {\em Explore Subspace Then
  Refine} (ESTR), an approach that achieves a regret bound of
$\tilO((d_1 + d_2)^{3/2}\sqrt{rT})$.  ESTR consists of two stages.  In
the first stage, we estimate the row and column subspace by randomly
sampling from a subset of arms, chosen carefully.  In the second stage, we leverage the estimated
subspace by invoking an approach  called {\em
  almost-low-dimensional OFUL} (LowOFUL), a variant of
OFUL~\cite{ay11improved} that uses regularization to penalize the
subspaces that are apparently {\em not} spanned by the rows and
columns (respectively) of $\Th^*$.
We conjecture that our regret upper bound is minimax optimal up to polylogarithmic factors based on the fact that the bilinear model has a much lower expected signal strength than the linear model.
We provide a detailed argument on the lower bound in Section~\ref{sec:lowerbound}.

While the idea of having an explicit exploration stage, so-called Explore-Then-Commit (ETC), is not new, the way we exploit
the subspace with LowOFUL is novel for two reasons. 
First, the standard ETC commits to the estimated parameter without refining and is thus known to have $\cO(\sqrt{T})$ regret only for ``smooth'' arm sets such as the unit ball~\cite{rusmevichientong10linearly,ay09forced}.
This means that the estimate refining is necessary for generic arm sets.
Second, after the first stage that
outputs a subspace estimate, it is tempting to project all the arms
onto the identified subspaces ($r$ dimensions for each row and column
space), and naively invoke OFUL in the $r^2$-dimensional space.
However, the subspace mismatch invalidates the upper confidence bound
used in OFUL; i.e., the confidence bound does not actually bound the mean reward.

Attempts to correct the confidence bound so that it is faithful are
not trivial, and we are unaware of a solution that leads to improved
regret bounds.  Departing from completely committing to the identified
subspaces, LowOFUL works with the full $d_1 d_2$-dimensional space,
but penalizes the subspace that is \emph{complementary to the
  estimated subspace}, thus continuing to \textit{refine} the
subspace.  We calibrate the amount of regularization to be a function
of the subspace estimation error; this is the key to achieving our
final regret bound.

We remark that our bandit problem can be modified slightly for the
setting in which the arm $\z_t$ is considered as a context, obtained
from the environment. This situation arises, for example, in
recommendation systems where $\cZ$ is the set of users represented by
indicator vectors (i.e., $d_2=N_2$) and $\cX$ is the set of items.
Such a setting is similar to~\citet{cb13agang}, but we assume that
$\Th^*$ is low-rank rather than knowing the graph information.
Furthermore, when the user information is available, one can take
$\cZ$ as the set of user feature vectors.

The paper is structured as follows.  In Section~\ref{sec:prelim}, we
define the problem formally and provide a sketch of the main
contribution.  Sections~\ref{sec:stage1} and~\ref{sec:stage2} describe
the details of stages 1 and 2 of ESTR, respectively.  We elaborate our
conjecture on the regret lower bound in Section~\ref{sec:lowerbound}.
After presenting our preliminary experimental results in
Section~\ref{sec:expr}, we discuss related work in
Section~\ref{sec:related} and propose future research directions in
Section~\ref{sec:conclusion}.

\begin{figure*}
  \fbox{
    \parbox[c]{.975\textwidth}{
      \textbf{Input}: time horizon $T$, the exploration length $T_1$, the rank $r$ of $\Th^*$, and the spectral bounds $S_F$, $S_2$, and $S_r$ of $\Th^*$.
      \vspace{.5em}\\%
      \textbf{Stage 1} (Section~\ref{sec:stage1})
      \begin{itemize}
        \item Solve (approximately)
        \begin{align} \label{eq:stage1-optim}
        \arg \max_{\text{distinct }\x^{(1)}, \ldots, \x^{(d_1)} \in \cX} \quad \lt(\text{the smallest eigenvalue of }  \lt[\x^{(1)}, \ldots, \x^{(d_1)}\rt] \rt)
        \end{align}
        and define  $\X = \{\x^{(1)},\cdots,\x^{(d_1)}\}$.
        Define $\Z$ similarly.
        \item For $T_1$ rounds, choose a pair of arms from $\X \times
        \Z$, pulling each pair the same number of times to the extent possible.
        That is, choose each pair $\lfloor \fr{T_1}{d_1 d_2}\rfloor$ times,
        then choose $T_1 - d_1 d_2 \lfloor \fr{T_1}{d_1 d_2}\rfloor$
        pairs uniformly at random without replacement.
        \item Let $\wtil\K$ be a matrix such that $\wtil K_{ij}$ is the average reward of pulling the arm $(\x^{(i)},\z^{(j)})$.
        Invoke a noisy matrix recovery algorithm (e.g., OptSpace~\cite{keshavan10matrix}) with $\wtil\K$ and the rank $r$ to obtain an estimate $\hK$.
        \item Let $\hTh = \X^{-1} \hK (\Z^\T)^{-1} $ where $\X=[(\x^{(1)})^\T;\  \cdots\ ; (\x^{(d_1)})^\T]\in\RR^{d_1\times d_1}$ (abusing notation) and $\Z$ is defined similarly.
        \item Let $\hTh = \hU \hS \hV^\T$ be the SVD of $\hTh$. Let $\hUp$ and $\hVp$ be orthonormal bases of the complementary subspaces of $\hU$ and $\hV$, respectively.
        \item Let $\gam(T_1)$ be the subspace angle error bound such that, with high probability,
        \begin{align}\label{eq:stage1-angle}
        \|\hUp^\T \U^* \|_F \|\hVp^\T \V^*\|_F \le \gam(T_1)
        \end{align}
        where $\Th^* = \U^* \S^* \V^{*\T}$ is the SVD of $\Th^*$.
      \end{itemize}
      \vspace{.5em}
      \textbf{Stage 2} (Section~\ref{sec:stage2})
      \begin{itemize}
        \item Rotate the arm sets: $\cX' = \left\{[\hU \hUp]^\T \x: \x \in \cX\right\}$ and $\cZ' = \left\{[\hV \hVp]^\T\z: \z \in \cZ\right\}$.
        \item Define a vectorized arm set so that the last $(d_1-r)\cdot(d_2-r)$ components are from the complementary subspaces:
        $$\cA = \left\{ [ \vec(\x_{1:r}\z_{1:r}^\T); \vec(\x_{r+1:d_1}\z_{1:r}^\T); \vec(\x_{1:r}\z_{r+1:d_2}^\T); \vec(\x_{r+1:d_1}\z_{r+1:d_2}^\T) ] \in \RR^{d_1d_2}: \x \in \cX', \, \z \in \cZ'  \right\}\ .$$
        \item For $T_2 = T-T_1$ rounds, invoke LowOFUL with the arm set $\cA$, the low dimension $k = (d_1 + d_2)r - r^2$, and $\gam(T_1)$. 
      \end{itemize}
  }}
  \vspace{-.5em}
  \caption{A sketch of Explore Subspace Then Refine (ESTR)}
  \label{fig:estr}
  \vspace{-1em}
\end{figure*}

\section{Preliminaries}
\label{sec:prelim}

We define the problem formally as follows.  Let
$\cX\subseteq\RR^{d_1}$ and $\cZ\subseteq \RR^{d_2}$ be the left and
right arm space, respectively.  Define $N_1 = |\cX|$ and $N_2 =
|\cZ|$. (Either or both can be infinite.)  We assume that both the
left and right arms have Euclidean norm at most 1: $\|\x\|_2 \le 1$
and $\|\z\|_2 \le 1$ for all $\x\in\cX$ and $\z\in\cZ$.  Without loss
of generality, we assume $\cX$ ($\cZ$) spans the whole $d_1$ ($d_2$)
dimensional space (respectively) since, if not, one can project the arm set to a lower-dimensional space that is now fully spanned.\footnote{
  In this case, we effectively work with a projected version of $\Th^*$, and its rank may become smaller as well.
}
We assume $d_2 = \Theta(d_1)$ and
define $d = \max\{d_1,d_2\}$.  If $A$ is a positive integer, we use
notation $[A] = \{1,2,\ldots,A\}$.  
We denote by $\v_{i:j}$ the
$(j-i+1)$-dimensional vector taking values from the coordinates from
$i$ to $j$ from $\v$.  Similarly, we define $\M_{i:j,k:\ell} \in
\RR^{(j-i+1) \times (\ell-k+1)}$ to be a submatrix taking values from
$\M$ with the row indices from $i$ to $j$ and the column indices from
$k$ to $\ell$. 
We denote by $v_i$ the $i$-th component of the vector $\v$ and by $M_{ij}$ the entry of a matrix $\M$ located at the $i$-th row and $j$-th column.
Denote by $\Sigma_k(\M)$ the $k$-th largest singular value, and define $\Sigma_{\max}(\M) = \Sigma_1(\M)$.
Let $\Sigma_{\min}(\M)$ be the smallest nonzero singular value of $\M$.
$|\M|$ denotes the determinant of a matrix $\M$.

The protocol of the bilinear bandit problem is as follows.  At time
$t$, the algorithm chooses a pair of arms $(\x_t,\z_t) \in \cX \times
\cZ$ and receives a noisy reward $y_t$ according to
\eqref{eq:noise-model}.  We make the standard assumptions in linear
bandits: the Frobenius and operator norms of $\Th^*$ are bounded by
known constants, $\|\Th^*\|_F \le S_F$ and $\|\Th^*\|_2 \le
S_2$,\footnote{
  When $S_2$ is not known, one can set $S_2 = S_F$. In
  some applications, $S_2$ is known. 
  For example, the binary model $y_t \sim \text{Bernoulli}((\x_t^\T\Th^* \z_t)+1)/2)$, we can evidently set $S_2=1$.
}  
and the sub-Gaussian scale $\sig$ of $\eta_t$ is known to the algorithm.  We
denote by $s^*_i$ the $i$-th largest singular value of $\Th^*$.  We
assume that the rank $r$ of the matrix is known and that $s^*_r \ge
S_r$ for some known $S_r>0$.~\footnote{In practice, one can perform
  rank estimation after the first stage
  (see, for example,~\citet{keshavan10matrix}).}

The main contribution of this paper is the first nontrivial upper
bound on the achievable regret for the bilinear bandit problem.  In
this section, we provide a sketch of the overall result and the key
insight.  For simplicity, we omit constants and variables other
than $d$, $r$, and $T$.  Our proposed ESTR algorithm enjoys the
following regret bound, which strictly improves the naive linear
bandit reduction when $r \ll d$.
\begin{thm}[An informal version of Corollary~\ref{cor:estr-regret}]~\label{thm:informal}
  Under mild assumptions, the regret of ESTR is
  $\tilO(d^{3/2}\sqrt{rT})$ with high probability.
\end{thm}
We conjecture that the regret bound above is minimax optimal up to polylogarithmic factors since the expected signal strength in the bilinear model is much weaker than the linear model.
We elaborate on this argument in Section~\ref{sec:lowerbound}.

We describe ESTR in Figure~\ref{fig:estr}.  The algorithm proceeds in
two stages.  In the first stage, we estimate the column and row
subspace of $\Th^*$ from noisy rank-one measurements using a matrix
recovery algorithm.
Specifically, we first identify $d_1$ and $d_2$ arms from the set
$\cX$ and $\cZ$ in such a way that the smallest singular values of the
matrices formed from these arms are maximized approximately (see~\eqref{eq:stage1-optim}), which is a form of submatrix selection problem (details in Section~\ref{sec:stage1}).
We emphasize that finding the exact solution is not necessary here since Theorem~\ref{thm:informal} has a mild dependency on the smallest eigenvalue found when approximating~\eqref{eq:stage1-optim}.
We then use the popular matrix recovery algorithm,
OptSpace~\cite{keshavan10matrix} to estimate $\Th^*$.  The $\sin
\Theta$ theorem of Wedin~\cite{stewart90matrix} is used to convert the
matrix recovery error bound from OptSpace to the desired subspace
angle guarantee~\eqref{eq:stage1-angle} with $ \gam(T_1) = \cO\lt(\fr{d^3
  r}{T_1}\rt) $.  The regret incurred in stage 1 is bounded trivially
by $T_1 \|\Th^*\|_2$.

In the second stage, we transform the problem into a $d_1
d_2$-dimensional linear bandit problem and invoke LowOFUL that we introduce in Section~\ref{sec:stage2}.
This technique projects the arms onto both the estimated subspace and its
complementary subspace and uses $\gam(T_1)$ to penalize weights in the
complementary subspaces $\hUp$ and $\hVp$.  LowOFUL enjoys regret
bound $\tilO( (dr + \sqrt{T} \gam(T_1)) \sqrt{T-T_1} )$ during $T - T_1$
rounds.  By combining with the regret for the first stage, we obtain
an overall regret of
\begin{align*}
T_1 + \lt(dr + \sqrt{T} \fr{d^3r}{T_1}\rt) \sqrt{T}.
\end{align*}
Choosing $T_1$ to minimize this expression, we obtain a regret bound
of $\tilO( d^{3/2} \sqrt{rT} )$.

\section{Stage 1: Subspace estimation}
\label{sec:stage1}

The goal of stage 1 is to estimate the row and column subspaces for
the true parameter $\Th^*$.  How should we choose which arm pairs to
pull, and what guarantee can we obtain on the subspace estimation
error?
One could choose to apply a noisy matrix recovery algorithm with
affine rank minimization~\cite{recht10guaranteed,mohan10new} to the
measurements attained from the arm pulls.
However, these methods require the measurements to be Gaussian or
Rademacher, so their guarantees depend on satisfaction of a RIP
property~\cite{recht10guaranteed}, or, for rank-one projection measurements, an RUB property~\cite{cai15rop}.
Such assumptions are not suitable
for our setting since measurements are restricted to the arbitrarily given arm
sets $\cX$ and $\cZ$.  Uniform sampling from the arm set cannot
guarantee RIP, as the arm set itself can be heavily biased in certain
directions.

We design a simple reduction procedure though matrix recovery with
noisy entry observations, leaving a more sophisticated treatment as future work.
The $d_1$ arms in $\cX$ are chosen according to the criterion
\eqref{eq:stage1-optim}, which is a combinatorial problem that is hard
to solve exactly. Our analysis does not require its exact solution,
however; it is enough that the objective value is nonzero (that
is, the matrix $\X$ constructed from these $d_1$ arms is
nonsingular). (Similar comments hold for the matrix $\Z$.)  
We remark that the problem~\eqref{eq:stage1-optim} is shown to be NP-hard by~\citet{civril09selecting} and is related to finding submatrices with favorable spectral properties~\cite{civril07finding,tropp09column}, but a thorough review on algorithms and their limits is beyond the scope of the paper.
For our experiments, simple methods such as random selection were sufficient; we describe our implementation in the supplementary material.

If $\K^*$ is the matrix defined by $K^*_{ij} = \x^{(i) \T} \Th^* \z^{(j)}$, each
time step of stage 1 obtains a noisy estimate of one element of $\K^*$.
Since multiple measurements of each entry are made, in general, we
compute average measurements for each entry.  A matrix recovery
algorithm applied to this matrix of average measurements yields the
estimate $\what\K$ of the rank-$r$ matrix $\K^*$.  Since $\K^* = \X
\Th^* \Z^\T$, we estimate $\Th^*$ by $\hTh = \X^{-1} \what\K
(\Z^\T)^{-1}$ and then compute the subspace estimate $\hU\in\RR^{d_1
  \times r}$ and $\hV\in\RR^{d_2 \times r}$ by applying SVD to
$\hTh$.

We choose the recovery algorithm OptSpace by~\citet{keshavan10matrix}
because of its strong (near-optimal) guarantee.
Denoting the SVD of $\K^*$ by $\U \R \V^\T$, we use the matrix
incoherence definition from~\citet{keshavan10matrix} and let
$(\mu_0,\mu_1)$ be the smallest values such that for all $i \in [d_1], j\in[d_2]$,
\begin{align*}
  \sum_{k=1}^r U_{ik}^2 \le \mu_0 r/d_1 , \quad  \sum_{k=1}^r V_{jk}^2 \le \mu_0 r /d_2, \quad \text{ and }
  \\ \lt| \sum_{k=1}^r U_{ik} (\Sigma_{k}(\K^*) / \Sigma_{\max}(\K^*)) V_{jk} \rt| \le \mu_1 \sqrt{\fr{r}{d_1d_2}} \ .
\end{align*}
Define the condition number $\kappa = \Sigma_{\max}(\K^*) /
\Sigma_{\min}(\K^*)$.  We present the guarantee of
OptSpace~\cite{keshavan10matrix} in a paraphrased form.  (The proof of
this result, and all subsequent proofs, are deferred to the supplementary material.)
\begin{thm}\label{thm:optspace-ours}
  There exists a constant $C_0$ such that for $T_1 \ge C_0 \sig^2(\mu_0^2
  + \mu_1^2) \fr{\kappa^6}{\Sigma_{\min}(\K^*)^2} dr(r + \log d)$, we
  have that, with probability at least $1 - 2/d_2^3$,
  \begin{align}\label{eq:thm-optspace-ours}
  \|\hK - \K^*\|_F  \le C_1 \kappa^2 \sig \fr{d^{3/2}\sqrt{r}}{\sqrt{T_1}}
  \end{align}
  where $C_1$ is an absolute constant.
\end{thm}
The original theorem from~\citet{keshavan10matrix} assumes
$T_1 \le d_1d_2$ and does not allow repeated sampling.
However, we show in the proof that the same guarantee holds for $T_1 > d_1d_2$ since repeated
sampling of entries has the effect of reducing the noise parameter
$\sig$.

Our recovery of an estimate $\hK$ of $\K^*$ implies the bound $\| \hTh
- \Th^*\|_F \le \|\X^{-1}\|_2 \|\Z^{-1}\|_2 \tau$ where $\tau$ is the
RHS of~\eqref{eq:thm-optspace-ours}.  However, our goal in stage 1
is to obtain bounds on the subspace estimation errors.  That is, given
the SVDs $\hTh = \hU \what{\mathbf{S}} \hV^\T$ and $\Th^* =
\U^* \S^* \V^{*\T}$, we wish to identify how
close $\hU$ ($\hV$) is to $\U^*$ ($\V^*$ respectively).
Such guarantees on the
subspace error can be obtained via the  $\sin \Theta$ theorem
by~\citet{stewart90matrix}, which we restate in our supplementary material.  Roughly, this theorem bounds the canonical angles between two
subspaces by the Frobenius norm of the difference between the two
matrices.  Recall that $s^*_r$ is the $r$-th largest singular value of
$\Th^*$.
\begin{thm}\label{thm:sintheta-ours} 
  Suppose we invoke OptSpace to compute $\hK$ as an estimate of the
  matrix $\K^*$.  After stage 1 of ESTR with $T_1$ satisfying the
  condition of Theorem~\ref{thm:optspace-ours}, we have, with probability at
  least $1-2/d_2^3$,
  \begin{align}
  \|\hUp^\T \U^* \|_F \|\hVp^\T \V^*\|_F \le \fr{\|\X^{-1}\|_2^2 \|\Z^{-1}\|_2^2 }{(s^*_r)^2}\tau^2
  \end{align}
  where $\tau = C_1 \kappa^2 \sig {d^{3/2}\sqrt{r}}/{\sqrt{T_1}}$.
\end{thm}


\section{Stage 2: Almost-low-dimensional linear bandits}
\label{sec:stage2}

The goal of stage 2 is to exploit the subspaces $\hU$ and $\hV$
estimated in stage 1 to perform efficient bandit learning.  At first,
it is tempting to project all the left and right arms to
$r$-dimensional subspaces using $\hU$ and $\hV$, respectively, which seems to be a bilinear bandit problem with an $r$ by $r$ unknown matrix.
One can then reduce it to an $r^2$-dimensional linear bandit problem and solve it by standard algorithms such as OFUL~\cite{ay11improved}.
Indeed, if $\hU$ and $\hV$ {\em exactly}
span the row and column spaces of $\Th^*$, this strategy yields a
regret bound of $\tilO(r^2\sqrt{T})$.  In reality, these matrices (subspaces) are
not exact, so there is model mismatch, making it difficult to apply
standard regret analysis. 
The upper confidence bound (UCB) used in popular algorithms becomes
invalid, and there is no known correction that leads to a regret bound
lower than $\tilO(d_1d_2\sqrt{T})$, to the best of our knowledge.

In this section, we show how stage 2 of our approach avoids the
mismatch issue by returning to the full $d_1d_2$-dimensional space,
allowing the subspace estimates to be inexact, but penalizing those
components that are complementary to $\hU$ and $\hV$. 
This effectively constrains the hypothesis space to be much smaller than the full $d_1d_2$-dimensional space.
We show how the bilinear bandit problem with good subspace estimates can be turned into the
\emph{almost low-dimensional linear bandit problem}, and how much
penalization / regularization is needed to achieve a low overall
regret bound.  Finally, we state our main theorem showing the overall
regret bound of ESTR.

\vspace{-.5em}
\paragraph{Reduction to linear bandits.}
Recall that $\Th^* = \U^* \S^* {\V^*}^\T$ is the SVD of $\Th^*$ (where
$\S^*$ is $r \times r$ diagonal) and that $\hUp$ and $\hVp$ are the
complementary subspace of $\hU$ and $\hV$ respectively.  Let ${\M} =
[\hU\; \hUp]^\T \Th^* [\hV\; \hVp]$ be a rotated version of $\Th^*$.
Then we have
\begin{align*}
  \Th^* &= [\hU\; \hUp] \M [\hV\; \hVp]^\T  \quad \text{ and }
  \\  \x^\T  \Th^* \z &= ([\hU\; \hUp]^\T\x)^\T \M ([\hV\;\hVp]^\T \z) \;.
\end{align*}
Thus, the bilinear bandit problem with the unknown $\Th^*$ with arm
sets $\cX$ and $\cZ$ is equivalent to the one with the unknown $\M$
with arm sets $\cX' = \{\x' = [\hU\; \hUp]^\T \x \mid \x\in\cX\}$ and
$\cZ'$ (defined similarly).  As mentioned earlier, this problem can be
cast as a $d_1 d_2$-dimensional linear bandit problem by considering
the unknown vector $\th^* = \vec(\M)$.  The difference is, however,
that we have learnt something about the subspace in stage 1.  We
define $\th^*$ to be a rearranged version of $\vec(\M)$ so that the
last $(d_1-r)\cdot(d_2-r)$ dimensions of $\th^*$ are $M_{ij}$ for $i \in \{r+1,\ldots,d_1\}$ and $j \in \{r+1,\ldots,d_2\}$; that is, letting $k := d_1d_2 - (d_1-r)\cdot (d_2-r)$,
\begin{subequations} \label{eq:def-thstar}
\begin{align*}
&\th^*_{1:k} = [\vec(\M_{1:r,1:r});\  \vec(\M_{r+1:d_1, 1:r});\  \vec(\M_{1:r, r+1:d_2})],
\\&\th^*_{k+1:p} = \vec(\M_{r+1:d_1,r+1:d_2})~.
\end{align*}
\end{subequations}
Then we have
\begin{align}\label{eq:thstar-tail-bound}
\begin{split}
\|\th^*_{k+1:p}\|_2^2
  &= \sum_{i>r \wedge j>r} M^2_{ij} 
    = \|\hUp^\T (\U^*\S^*\V^{*\T} )\hVp \|^2_F
\\&\le  \|\hUp^\T \U^*\|^2_F \|\S^*\|^2_2 \|\hVp^\T \V^*\|^2_F \ ,
\end{split}
\end{align}
which implies $\|\th^*_{k+1:p}\|_2 = \cO(d^3r/T_1)$ by Theorem~\ref{thm:sintheta-ours}.
Our knowledge on the subspace results in the knowledge of the norm of
certain coordinates!  Can we exploit this knowledge to enjoy a better
regret bound than $\tilO(d_1d_2\sqrt{T})$?  We answer this question in
the affirmative below.

\vspace{-.5em}
\paragraph{Almost-low-dimensional OFUL (LowOFUL).}
We now focus on an abstraction of the conversion described in the previous paragraph, which we call the \textit{almost-low-dimensional linear bandit problem}.
In the standard linear bandit problem in $p$ dimensions, the player chooses an arm $\a_t$ at time $t$ from an arm set $\cA\subseteq\RR^p$ and observes a noisy reward $y_t = \la \a_t ,\th^* \ra + \eta_t $, where the noise $\eta_t$ has the same properties as in \eqref{eq:noise-model}.  
We assume that $\|\a\|_2 \le 1$ for all $\a\in\cA$, and $\|\th^*\|_2 \le B$ for some known constant $B>0$.
In almost-low-dimensional linear bandits, we have additional knowledge that $\|\th_{k+1:p}^*\|_2 \le B_\perp$ for some index $k$ and some constant $B_\perp$ (ideally $\ll B$).
This means that all-but-$k$ dimensions of $\th^*$ are close to zero.

To exploit the extra knowledge on the unknown, we propose \textit{almost-low-dimensional OFUL} (LowOFUL) that extends the standard linear bandit algorithm OFUL~\cite{ay11improved}.
To describe OFUL, define the design matrix
$\A \in \RR^{t \times p}$ with rows $\a_s^\T$, $s = 1,2,\dotsc,t$ and
the vector of rewards $\y = [y_1,\ldots,y_t]^\T$.  
The key estimator is based on regression with the standard squared $\ell_2$-norm regularizer, as follows:
\begin{align}\label{eq:def-hth-oful}
  \hth_t = \arg \min_{\th} \, \fr{1}{2} \| \A \th - \y \|_2^2 + \fr{\lam}{2}\|\th\|_2^2 = (\lambda\I + \A^\T\A)^{-1} \A^\T \y \;.
\end{align}
OFUL then defines a confidence ellipsoid around $\hth_t$ based on
which one can compute an upper confidence bound on the mean reward of
any arm.
In our variant, we allow a different regularization for each
coordinate, replacing the regularizer $\fr{\lambda}{2}\|\th\|_2^2$ by
$\fr{1}{2}\| \th \|^2_{\Lam} = \fr{1}{2} \th^\top \Lam \th$ for some
positive diagonal matrix $\Lam$.  Specifically, we define $\Lam =
\diag( \lam,\ldots,\lam, \lamp, \ldots, \lamp )$, where $\lam$
occupies the first $k$ diagonal entries and $\lamp$ the last $p-k$
positions. With this modification, the estimator becomes
\begin{align}\label{eq:def-hth}
   \hth_t = \arg \min_{\th} \fr{1}{2} \| \A \th - \y \|_2^2 +
   \fr{1}{2}\|\th\|_{\Lam}^2 = (\Lam + \A^\T\A)^{-1} \A^\T \y \;.
\end{align}
Define $\V_t = \Lam + \sum_{s=1}^t \a_t \a_t^\T = \Lam + \A^\T \A$ and
let $\delta$ be the failure rate we are willing to endure.  The
confidence ellipsoid for $\th^*$ is
\begin{align}\label{eq:def-ct}
\begin{split}
  c_t &= \left\{\th: \|\th - \hth_t\|_{\V_t} \le \sqrt{\beta_t} \right\}  \ \ \ \text{where}
\\\sqrt{\beta_t} &= \sig\sqrt{\log\fr{|\V_t|}{|\Lam|\delta^2}} + \sqrt{\lam}B + \sqrt{\lamp} B_\perp.
\end{split}
\end{align}
This ellipsoid enjoys the following guarantee, which is a direct
consequence of~\citet[Lemma 3]{valko14spectral} that is based on the
self-normalized martingale inequality of~\citet[Theorem
  1]{ay11improved}.
\begin{lem}\label{lem:lowoful-concentration}
  With probability at least $1-\dt$, we have $\th^* \in c_t$ for all
  $t \ge 1$.
\end{lem}
We summarize LowOFUL in Algorithm~\ref{alg:lowoful}, where $\max_{\th
  \in c_{t-1}} \la \th, \a\ra$ can be simplified to $\la \hth_{t-1},
\a \ra + \sqrt{\beta_{t-1}} \| \a \|_{\V_{t-1}^{-1}}$.

\begin{algorithm}[t]
  \begin{algorithmic}[1]
    \STATE \textbf{Input}: $T$, $k$, the arm set $\cA\subseteq\RR^p$, failure rate $\dt$, and positive constants $B$, $B_\perp$, $\lambda$, $\lambda_\perp$.
    \STATE Set $\Lam = \diag( \lam,\ldots,\lam, \lamp, \ldots, \lamp )$ \,  where $\lam$ occupies the first $k$ diagonal entries.
    \FOR {$t=1,2,\ldots,T$}
    \STATE Compute $\a_t = \arg \max_{\a \in \cA} \max_{\th \in c_{t-1}} \la \th, \a \ra$.
    \STATE Pull arm $\a_t$.
    \STATE Receive reward $y_t$.
    \STATE Set $c_t$ as~\eqref{eq:def-ct}. 
    \ENDFOR 
  \end{algorithmic}
  \caption{LowOFUL}
  \label{alg:lowoful}
\end{algorithm}

We now state the regret bound of LowOFUL in
Theorem~\ref{thm:lowoful-regret}, which is based on the standard
linear bandit regret analysis dating back to~\citet{auer02using}.
\begin{thm}\label{thm:lowoful-regret}
 The regret of LowOFUL is, with probability at least $1-\dt$,
\begin{align}\label{eq:oful-regret-bound}
\cO\lt(\sqrt{\log \fr{|\V_T|}{|\Lam|}} \lt( \sig\sqrt{\log
  \fr{|\V_T|}{|\Lam|\dt^2}} + \sqrt{\lam} B + \sqrt{\lamp} B_\perp
\rt)\cdot \sqrt{T}\rt) \;.
\end{align}
\end{thm}
In the standard linear bandit setting where $\lamp=\lam$ and
$B_\perp=B$, we recover the regret bound $\tilO(p\sqrt{T})$ of OFUL,
since $\log \fr{|\V_T|}{|\Lam|} =
\cO(p\sqrt{T})$~\cite[Lemma~10]{ay11improved}.

To alleviate the dependence on $p$ in the regret bound, we propose a
carefully chosen value of $\lamp$ in the following corollary.
\begin{cor}\label{cor:lowoful-regret}
 Then, the regret of LowOFUL with $\lamp =
 \fr{T}{k\log(1+\fr{T}{\lam})}$ is, with probability at least $1-\dt$,
 \begin{align*}
    \tilO\lt( (\sig k + \sqrt{k\lambda}B + \sqrt{T}B_\perp) \sqrt{T}\rt)\ .
 \end{align*}
\end{cor}
The bound improves the dependence on dimensionality from $p$ to $k$,
but introduces an extra factor of $\sqrt{T}$ to $B_\perp$, resulting in
linear regret.  
While this choice is not interesting in general, we show that it is useful for our case: 
Since $\|\th^*_{k+1:p}\|_2 = \cO(1/T_1)$, we can set $B_\perp = \cO(1/T_1)$ to be a valid upper bound of $\|\th^*_{k+1:p}\|_2$.
By setting $T_1 = \Theta(\sqrt{T})$, the regret bound in
Corollary~\ref{cor:lowoful-regret} scales with $\sqrt{T}$ rather than $T$.


Concretely, using \eqref{eq:thstar-tail-bound}, we set 
\begin{align}\label{eq:def-B-Bp}
\begin{split}
B &= S_F \quad\text{ and } \quad  B_\perp = S_2 \cdot \gam(T_1)  \quad \text{ where }
\\\gam(T_1) &= \fr{\|\X^{-1}\|_2^2
  \|\Z^{-1}\|_2^2 }{(S_r)^2}\cdot C_1^2 \lt(\fr{S_2}{S_r}\rt)^4 \sig^2.
\fr{d^3 r}{T_1}~.
\end{split}
\end{align}
$B$ and $B_\perp$ are valid upper bounds of $\|\th^*\|_2$ and $\|\th^*_{k+1:p}\|_2$, respectively, with high probability.
Note we must use $S_2$, $S_r$, and $S_2/S_r$ instead of $s^*_1$, $s^*_r$,
and $\kappa$, respectively, since the latter variables are unknown to the
learner.

\vspace{-.6em}
\paragraph{Overall regret.}
Theorem~\ref{thm:estr-regret} shows the overall regret bound of ESTR.
\begin{thm} \label{thm:estr-regret}
  Suppose we run ESTR (Algorithm~\ref{alg:lowoful}) with $T_1 \ge C_0 \sig^2(\mu_0^2 + \mu_1^2)
  \fr{\kappa^6}{\Sigma_{\min}(\K^*)^2} dr(r + \log d)$.  We invoke
  LowOFUL in stage 2 with $p=d_1 d_2$, $k=r\cdot(d_1 + d_2 - r)$, $\th^*$
  defined as~\eqref{eq:def-thstar}, the rotated arm sets $\cX'$ and $\cZ'$,
  $\lamp = \fr{T_2}{k\log(1+T_2/\lambda)}$, and $B$ and $B_\perp$ as
  in~\eqref{eq:def-B-Bp}.  The regret of ESTR is, with probability at
  least $1 - \delta - 2/ d_2^3 $, 
  \begin{align*}
    \tilO\lt(  s^*_1 T_1 + T \cdot \fr{\|\X^{-1}\|_2^2 \|\Z^{-1}\|_2^2 (S_2^5/S_r^6) \sig^2 d^3 r}{T_1} \rt).
  \end{align*}
\end{thm}
One can see that there exists an optimal choice of $T_1$, which we
state in the following corollary.
\begin{cor}\label{cor:estr-regret}
  Suppose the assumptions in Theorem~\ref{thm:estr-regret} hold.  If
  $T_1 = \Theta\lt( \|\X^{-1}\|_2 \|\Z^{-1}\|_2 \fr{S_2^2}{S_r^3} \sig
  d^{3/2} \sqrt{rT} \rt) $, then the regret of ESTR is, with
  probability at least $1 - \delta - 2 / d_2^3 $, 
  \begin{align*}
    \tilO\lt( \fr{S_2^3}{S_r^3} \|\X^{-1}\|_2 \|\Z^{-1}\|_2 \sig d^{3/2}\sqrt{rT}  \rt) \ .
  \end{align*}
\end{cor}
Note that, for our problem, the incoherence constants $\mu_0$ and $\mu_1$ do not play an important role with large enough $T$.

\vspace{-.5em}
\paragraph{Remark}
One might notice that 
we can also regularize the submatrices $\M_{r+1:d_1, 1:r}$ and
$\M_{1:r,r+1:d_2}$ since they are coming partly from the complementary subspace of $\hU$
and partly from the complement of $\hV$ (but not both).  In practice,
such a regularization can be done to reduce the regret slightly, but
it does not affect the order of the regret.  We do not have sufficient
decrease in the magnitude to provide interesting bounds. One can show that, while $\|\M_{r+1:d_1,r+1:d_2}\|_F^2 = \cO(1/T_1)$, the quantities $\|\M_{1:r,r+1:d_2}\|_F^2$ and $\|\M_{r+1:d_1,1:r}\|_F^2$ are $\cO(1/\sqrt{T_1})$.


\section{Lower bound}
\label{sec:lowerbound}

A simple lower bound is $\Omega(d\sqrt{T})$, since when the arm set $\Z$ is a singleton the problem reduces to a $d_1$-dimensional linear bandit problem. We have attempted to extend existing lower-bound proof techniques in~\citet{rusmevichientong10linearly},~\citet{dani08stochastic}, and~\citet{lattimore18bandit}, but the bilinear nature of the problem introduces cross terms between the left and right arm, which are difficult to deal with in general.  However, we conjecture that the lower bound is $\Omega(d^{3/2}\sqrt{rT})$.  We provide an informal argument below that the dependence on $d$ must be $d^{3/2}$ based on the observation that the rank-one bilinear reward model's signal-to-noise ratio (SNR) is significantly worse than that of the linear reward model.

Consider a rank-one $\Th^*$ that can be decomposed as $\u \v^\T$ for some $\u,\v \in \{\pm 1/\sqrt{d}\}^d$.  Suppose the left and right arm sets are $\cX = \cZ = \{ \pm 1/\sqrt{d} \}^d$.  Let us choose $\x_t$ and $\z_t$ uniformly at random (which is the sort of pure exploration that must be performed initially).  Then a simple calculation shows that the expected squared signal strength with such a random choice is $\EE |\x_t^\T\Th^* \z_t|^2=\fr{1}{d^2}$.  In contrast, the expected squared signal strength for a linear reward model is $\EE |\x_t^\T \u|^2 = \fr{1}{{d}}$.  The effect of this is analogous to increasing the sub-Gaussian scale parameter of the noise $\eta_t$ by a factor of $\sqrt{d}$.  We thus conjecture that the $\sqrt{d}$ difference in the SNR introduces the dependence $d^{3/2}$ in the regret rather than $d$.

\section{Experiments}
\label{sec:expr}

\begin{figure}
\begin{center}
  \begin{tabular}{c}
    \includegraphics[width=.7\columnwidth,valign=t]{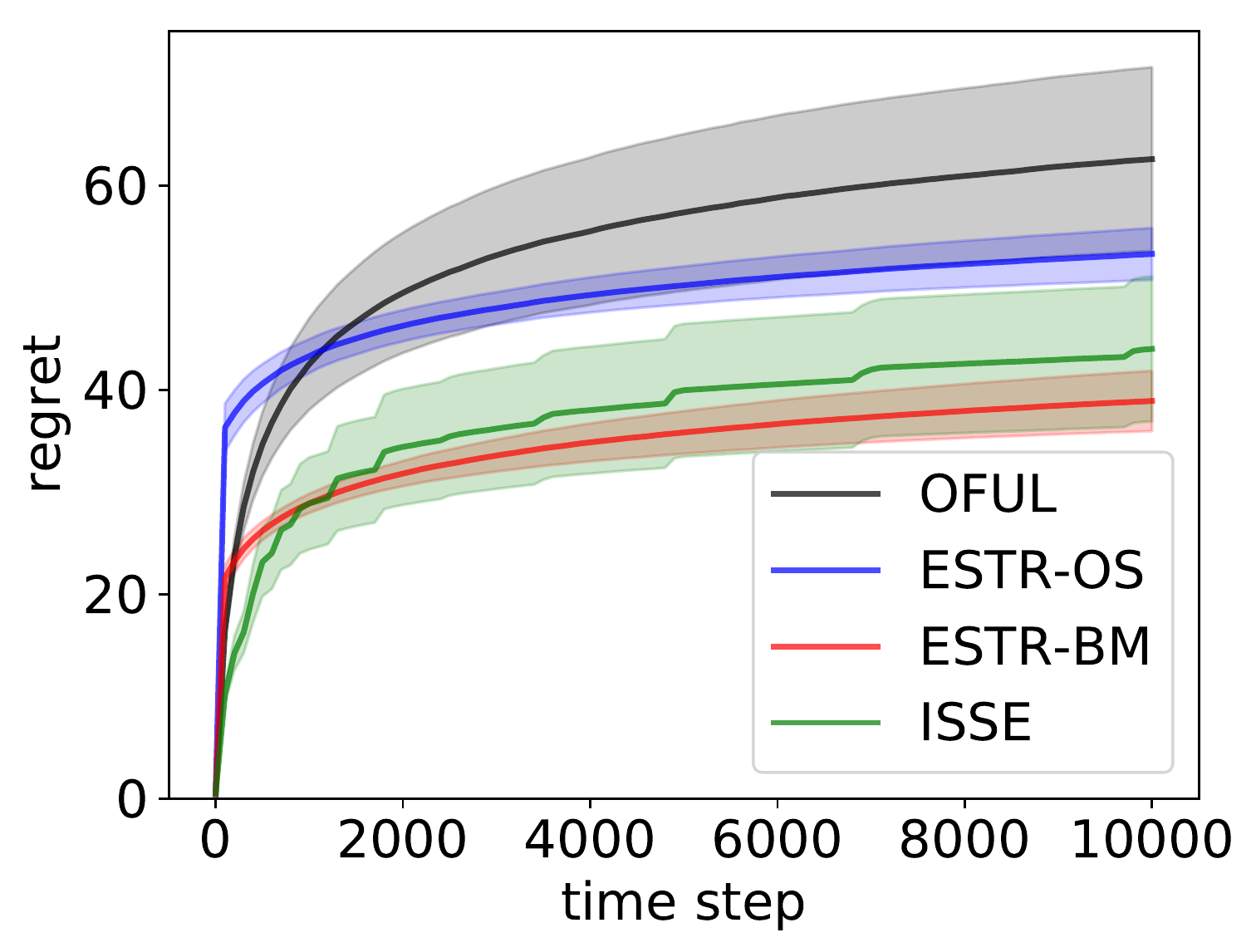} \\
  \end{tabular}
\end{center}
\vspace{-1em}
  \caption{
    Simulation results for $d=8$ and $r=1$.
    Our method ESTR-OS, its variant ESTR-BM, and an implicit exploration variant of ESTR called ISSE
    all outperform the baseline linear bandit method OFUL.
  }
  \label{fig:expr-toy}
\vspace{-1.5em}
\end{figure}
We present a preliminary experimental result and discuss practical concerns.

Bandits in practice requires tuning the exploration rate to perform well, which is usually done 
by adjusting the confidence bound width~\cite{chapelle11anempirical,li10acontextual,zhang16online}, which amounts to replacing $\beta_t$ with $c \beta_t$ for some $c>0$ for OFUL or its variants (including LowOFUL).
An efficient parameter tuning in bandits is an open problem and is beyond our scope.
For the sake of comparison, we tune $c$ by grid search and report the result with the smallest average regret.
For ESTR, the value of $T_1$ used in the proof involves some unknown constants; to account for this, we tune $T_1$ by grid search.
We consider the following methods:
\begin{itemize}
  \item OFUL: The OFUL reduction described in~\eqref{eq:oful-reduction}, which ignores the low-rank structure.
  \item ESTR-OS: Our proposed method; we simplify $B_\perp$ in~\eqref{eq:def-B-Bp} to $S_2 \sig^2 d^3 r / T_1$.
  \item ESTR-BM: We replace OptSpace with the Burer-Monteiro formulation and perform the alternating minimization~\cite{burer03anonlinear}.
  \item ISSE (Implicit SubSpace Exploration): LowOFUL with a heuristic subspace estimation that avoids an explicit exploration stage. 
     We split the time intervals with knots at $t \in \{10^{0.5}, 10^{1}, 10^{1.5}, ...\}$.
     At the beginning time $t'$ of each interval, we perform the matrix recovery with the Burer-Monteiro formulation using all the past data, estimate the subspaces, and use them to initialize LowOFUL with $B_\perp = S_2\sig^2 d^3 r / t'$ and all the past data.
\end{itemize}
Note that OFUL and ISSE only require tuning $c$ whereas ESTR methods require tuning both $c$ and $T_1$.

We run our simulation with $d_1=d_2=8$, $r=1$, $\sig=0.01$.
We set $\lambda=1$ for both OFUL and LowOFUL.
We draw 16 arms from the unit sphere for each arm set $\cX$ and $\cZ$ and simulate the bandit game for $T=10^4$ iterations, which we repeat 60 times for each method.
Figure~\ref{fig:expr-toy} plots the average regret of the methods and the .95 confidence intervals.
All the methods outperform OFUL, and the regret differences from OFUL are statistically significant.
We observe that ESTR-BM performs better than ESTR-OS.
We believe this is due to our limit on the number of iterations of OptSpace set to 1000, which we imposed due to its slow convergence in our experiments.\footnote{
  We used the authors' C implementation that is wrapped in python (\url{https://github.com/strin/pyOptSpace}).
}
The Burer-Monteiro formulation, however, converged within 200 iterations.
Finally, ISSE performs close to ESTR-BM, but with a larger variance.
Although ISSE does not have a theoretical guarantee, it does not require tuning $T_1$ and performs better than OFUL.

\section{Related work}
\label{sec:related}

There exist a few studies on pulling a pair of arms as a unit action,
as we do.
\citet{kveton17stochastic_arxiv} consider the $K$-armed
bandit with $N_1$ left arms and $N_2$ right arms.  The expected rewards
can be represented as a matrix $\bar \R \in \RR^{N_1 \times N_2}$
where the authors assume $\bar \R$ has rank $r \ll \min\{N_1, N_2\}$.
The main difference from our setting is that they do not assume that
the arm features are available, so our work is related
to~\citet{kveton17stochastic_arxiv} in the same way as the linear
bandits are related to $K$-armed bandits.  The problem considered in
\citet{katariya17stochastic} is essentially a rank-one version
of~\citet{kveton17stochastic_arxiv}, which is motivated by a click-feedback model called position-based
model with $N_1$ items and $N_2$ positions.  This work is further
extended to have a tighter KL-based bound
by~\citet{katariya17bernoulli}.  All these studies successfully
exploit the low-rank structure to enjoy regret bounds that scale with
$r(N_1 + N_2)$ rather than $N_1 N_2$. 
\citet{zimmert18factored} propose a more generic problem called factored bandits whose action set is a product of atomic $L$ action sets rather than two.
While they achieve generality by not require to know the explicit reward model, factored bandits do not leverage the known arm features nor the low-rank structure, resulting in large regret in our problem.

There are other works that exploit the low-rank structure of the
reward matrix, although the action is just a single arm pull.
\citet{sen17contextual} consider the contextual bandit setting where
there are $N_1$ discrete contexts and $N_2$ arms, but do not take into
account the observed features of contexts or arms.  Under the
so-called separability assumption, the authors make use of Hottopix
algorithm to exploit the low-rank structure.
\citet{gopalan16lowrank} consider a similar setting, but employ the robust tensor power method for recovery.
\citet{kawale15efficient} study essentially the same problem, but make
assumptions on the prior that generates the unknown matrix and
perform online matrix factorization with particle filtering to
leverage the low-rank structure. These studies also exploit the
low-rank structure successfully and enjoy regret bounds that scale
much better than $N_1 N_2$.

There has been a plethora of contextual bandit studies that exploit
structures other than the low-rank-ness, where the context is usually
the user identity or features.  For example, \citet{gentile14online}
and its followup studies~\cite{li16collaborative,gentile17oncontext}
leverage the clustering structure of the contexts.
In~\citet{cb13agang} and~\citet{vaswani17horde}, a graph structure of
the users is leveraged to enjoy regret bound that is lower than running
bandits on each context (i.e., user) independently.
\citet{deshmukh17multitask} introduce a multitask learning view and
exploit arm similarity information via kernels, but their regret
guarantee is valid only when the similarity is known ahead of time.
In this vein, if we think of the right arm set $\cZ$ as tasks, we
effectively assume different parameters for each task but with a
low-rank structure.  That is, the parameters can be written as a
linear combination of a few hidden factors, which are estimated
on the fly rather than being known in advance.
\citet{johnson16structured} consider low-rank structured bandits but in a different setup.
Their reward model has expected reward of the form $\tr(\X_t^\T \Th^*)$ with the arm $\X_t \in \RR^{d\times p}$ and the unknown $\Th^* \in \RR^{d \times p}$.
While $\X_t$ corresponds to $\x_t \z_t^\T$ in our setting, they consider a continuous arm set only, so their algorithm cannot be applied to our problem.


Our subroutine LowOFUL is quite similar to SpectralUCB
of~\cite{valko14spectral}, which is designed specifically for
graph-structured arms in which expected rewards of the two arms are
close to each other (i.e., ``smooth'') when there is an edge between
them.  Although technical ingredients for Corollary~\ref{cor:lowoful-regret} stem from~\citet{valko14spectral}, LowOFUL is for an inherently different setup in which we \emph{design} the regularization matrix $\Lam$ to maximally exploit the subspace
knowledge and minimize the regret, rather than receiving $\Lam$ from
the environment as a part of the problem definition.
\citet{gilton17sparse} study a similar regularizer in the context of
sparse linear bandits under the assumption that a superset of the
sparse locations is known ahead of time.  \citet{yue12hierarchical}
consider a setup similar to LowOFUL. They assume an estimate of the
subspace is available, but their regret bound still depends on the
total dimension $p$.

\section{Conclusion}
\label{sec:conclusion}

In this paper, we introduced the bilinear low-rank bandit problem and
proposed the first algorithm with a nontrivial regret guarantee. Our
study opens up several future research directions. First, there is
currently no nontrivial lower bound, and showing whether the regret of $\tilO(d^{3/2}\sqrt{rT})$ is
tight or not remains open.
Second, while our algorithm improves the regret bound over the trivial
linear bandit reduction, the algorithm requires to tune an extra parameter $T_1$.
It would be more natural to continuously update the subspace estimate and the amount of regularization, just like ISSE.
However, proving a theoretical guarantee would be challenging since most matrix recovery algorithms require some sort of uniform sampling with a ``nice'' set of measurements.  We speculate that one can employ randomized arm selection and use importance-weighted data to perform effective and provable matrix recoveries on-the-fly.


%
%


\section*{Acknowledgements}

This work was supported by NSF 1447449 - IIS, NIH 1 U54 AI117924-01, NSF CCF-0353079, NSF 1740707, and AFOSR FA9550-18-1-0166. 

\bibliography{library-shared}
\bibliographystyle{icml2019}

\clearpage
\onecolumn

\begin{center}
  {\Large\bf Supplementary Material}
\end{center}

\appendix

\section{Proof of Theorem~\ref{thm:optspace-ours}}
\label{sec:proof-keshavan}

\textbf{Theorem~\ref{thm:optspace-ours}} (Restated) \emph{ 
  There exists a constant $C_0$ such that for $T_1 \ge C_0 \sig^2(\mu_0^2
  + \mu_1^2) \fr{\kappa^6}{\Sigma_{\min}(\K^*)^2} dr(r + \log d)$, we
  have that, with probability at least $1 - 2/d_2^3$,
  \begin{align}
  \|\hK - \K^*\|_F  \le C_1 \kappa^2 \sig \fr{d^{3/2}\sqrt{r}}{\sqrt{T_1}}
  \end{align}
  where $C_1$ is an absolute constant.
}

\begin{proof}
There are a number of assumptions required for the guarantee of
OptSpace to hold. Given a noise matrix $\Z$, let $\tilK = \K^* + \Z$
be the noisy observation of matrix $\K^*$.  Among various noise models
in~\citet[Theorem~1.3]{keshavan10matrix}, the independent
$\rho$-sub-Gaussian model fits our problem setting well.  Let $E \in[m]\times[n]$ be the indicator
of observed entries and let $\Z^E$ be a censored version of $\Z$ in
which the unobserved entries are zeroed out.  Recall that we assume
$d_2 = \Theta(d_1)$, that $d = \max\{d_1, d_2\}$, and that $\kappa$ is
the condition number of $\K^*$.

We first state the guarantee and then describe the required technical
assumptions.  \citet[Theorem~1.2]{keshavan10matrix} states that the
following is true for some constant $C'>0$:
\begin{align*}
\|\hK - \K^*\|_F  \le C' \kappa^2 \fr{d^2\sqrt{r}}{|E|} \| \Z^E\|_2  \;.
\end{align*}
Here, by~\citet[Theorem 1.3]{keshavan10matrix}, $\|\Z^E\|_2$ is no
larger than $C'' \rho \sqrt{|E| / d}$, for some constant $C''>0$,
under Assumption~\textbf{(A3)} below, where $\rho$ is the sub-Gaussian
scale parameter for the noise $\Z$. ($\rho$ can be different from
$\sig$, as we explain below).  The original version of the statement has
a preprocessed version $\| \tilde \Z^E\|_2$ rather than $\| \Z^E\|_2$,
but they are the same under our noise model, according
to~\citet[Section 1.5]{keshavan10matrix}. Together, in our notation,
we have
\begin{align*}
\|\hK - \K^*\|_F  \le C' C'' \kappa^2 \rho \fr{d^{3/2}\sqrt{r}}{\sqrt{|E|}}.
\end{align*}
In the case of $T_1 < d_1 d_2$, the guarantee above holds true with
$\rho = \sig$ and $|E| = T_1$.  If $T_1 \ge d_1 d_2$, the guarantee holds
true with $\rho =
\sig\cdot\lt(\lt\lfloor\fr{T_1}{d_1d_2}\rt\rfloor\rt)^{-1/2} $ and $|E| =
d_1 d_2$.  In both cases, we arrive at~\eqref{eq:thm-optspace-ours}.

We now state the conditions.  Let $\alpha = d_1/d_2$.  Define
 ${\Sigma_{\min}}$ to be the smallest
nonzero singular values of $\M$.
\begin{itemize}
  \item (\textbf{A1}): $\M$ is $({\mu_0},{\mu_1})$-incoherent. Note
    $\mu_0 \in [1, \max\{d_1,d_2\}/r]$.
  \item (\textbf{A2}): (Sufficient observation) For some $C'>0$, we have
  \begin{align*}
  |E| \ge C' d_2\sqrt{\alpha} \kappa^2 \max \lt\{ \mu_0 r \sqrt{\alpha} \log d_2,\; \mu_0^2 r^2 \alpha \kappa^4,\; \mu_1^2 r^2 \alpha \kappa^4\rt\}\;,
  \end{align*}
  which we loosen and simplify to (using $\mu_0 \ge 1$)
  \begin{align*}
  |E| \ge C' \kappa^6(\mu_0^2 + \mu_1^2) dr(r + \log d).
  \end{align*}
  \item \textbf{(A3)}: $|E| \ge n \log n$. 
  \item \textbf{(A4)}: We combine the bound on $\|\Z^E\|_2$ and the
    condition in~\citet[Theorem 1.2]{keshavan10matrix} that says
    ``provided that the RHS is smaller than $\sigma_{\min}$'', which
    results in requiring
  \begin{align*}
  \fr{|E|}{\rho^2} \ge C'' \fr{\kappa^4}{\Sigma_{\min}^2} dr,
  \end{align*}
  for some $C''>0$, Using the same logic as before, either $T_1 <
  d_1d_2$ or $T_1 \ge d_1d_2$, we can rewrite the same statement in
  terms of $T_1$, as follows:
  \begin{align*}
  T_1 \ge C'' \sig^2 \fr{\kappa^4}{\Sigma_{\min}^2} dr \;.
  \end{align*}
\end{itemize}
All these conditions can be merged to 
\begin{align*}
T_1 \ge C_0 \sig^2 (\mu_0^2 + \mu_1^2) \fr{\kappa^6}{\Sigma_{\min}^2}
dr(r + \log d)
\end{align*}
for some constant $C_0>0$.
\end{proof}

\section{Proof of Theorem~\ref{thm:sintheta-ours}}

We first restate the  $\sin \Theta$ theorem due to
Wedin~\cite{stewart90matrix} in a slightly simplified form.
Let the SVDs of matrices $\A$ and $\tilde\A$ be defined as follows:
\begin{align*}
(\U_1\; \U_2\; \U_3)^\T \A (\V_1\; \V_2)  = \begin{pmatrix}
  \bfSigma_1 & \mathbf{0}
\\\mathbf{0} & \bfSigma_2
\\ \mathbf{0} & \mathbf{0}
\end{pmatrix},
\end{align*}
\begin{align*}
  (\tilde\U_1\; \tilde\U_2\; \tilde\U_3)^\T \tilde\A (\tilde\V_1\; \tilde\V_2)  = \begin{pmatrix}
  \tilde\bfSigma_1 & \mathbf{0}
  \\\mathbf{0} & \tilde\bfSigma_2
  \\ \mathbf{0} & \mathbf{0}
  \end{pmatrix}.
\end{align*}
Let $\R = \A\tilde \V_1 - \tilde \U_1 \tilde \bfSigma_1$ and $\S =
\A^\T \tilde \U_1 - \tilde\V_1\tilde\bfSigma_1$, and define
$\U_{1\perp} = [\U_2 \; \U_3]$ and $\V_{1\perp} = [\V_2\; \V_3]$.
Wedin's $\sin \Theta$ theorem, roughly speaking, bounds the sin
canonical angles between two matrices by the Frobenius norm of their
difference.
\begin{thm}[Wedin]
  Suppose that there is a number $\delta>0$ such that
  \begin{align*}
    \min_{i,j} \left| \sig_i \left( \tilde\bfSigma_1\right) - \sig_j\left(\bfSigma_2 \right) \right| \ge \delta
    \quad\text{ and }\quad
    \min_i  \sig_i \left(\tilde\bfSigma_1 \right) \ge \delta \;.
  \end{align*}
  Then,
  \begin{align*}
    \sqrt{\| \U_{1\perp}^\T \tilde \U_1 \|_F^2 + \| \V_{1\perp}^\T \tilde \V_1 \|_F^2} \le \fr{\sqrt{\|\R\|_F^2 + \|\S\|_F^2}}{\delta}
  \end{align*}
\end{thm}
We now prove Theorem~\ref{thm:sintheta-ours}.

\textbf{Theorem~\ref{thm:sintheta-ours}} (Restated) \emph{
  Suppose we invoke OptSpace to compute $\hK$ as an estimate of the
  matrix $\K^*$.  After stage 1 of ESTR with $T_1$ satisfying the
  condition of Theorem~\ref{thm:optspace-ours}, we have, with probability at
  least $1-2/d_2^3$,
  \begin{align}
  \|\hUp^\T \U^* \|_F \|\hVp^\T \V^*\|_F \le \fr{\|\X^{-1}\|_2^2 \|\Z^{-1}\|_2^2 }{(s^*_r)^2}\tau^2
  \end{align}
  where $\tau = C_1 \kappa^2 \sig {d^{3/2}\sqrt{r}}/{\sqrt{T_1}}$.
}
\begin{proof}
  In our case, the variables defined for Wedin's theorem are as follows:
  \begin{align*}
  \A   &= \hTh           & \til \A  &= \Th^*      \\
  \U_1 &= \hU           & \til\U_1 &= \U^*    \\
  \bfSigma_1 &= \hbS & \til\bfSigma_1 &= \S^*\\
  \V_1 &= \hV           & \til\V_1 &= \V^*    \;.
  \end{align*}
  Let ${\E} = \hTh - \Th^*$.
  Note that 
  \begin{align*}
  \R &= \hTh \V^* - \U^*\S^* = (\Th^* + \E)\V^* - \U^*\S^* =  \E \hV
  \\\S &= -\E^\T\U^* \quad(\text{similarly}) \;.
  \end{align*}
  Then, $\|\R\|_F = \|\E \hV\|_F \le \|\E\|_F$, using the fact that
  \[
  \|\E\|_F = \|\E [\hV \; \hVp ]\|_F = \sqrt{\| \E\hV \|_F^2 + \| \E\hVp \|_F^2}.
  \]
  Similarly, $\|\S\|_F \le \|\E\|_F$.
  We now apply the  $\sin \Theta$ theorem to obtain
  \begin{align*}
  \sqrt{2 \|\hUp^\T \U^* \|_F \|\hVp^\T \V^*\|_F} 
  \le \sqrt{    \|\hUp^\T \U^* \|_F^2 + \|\hVp^\T \V^*\|_F^2} 
  \le \fr{\sqrt{\|\R\|_F^2 + \|\S\|_F^2}}{ \dt } \le \fr{\sqrt{2 \|\E\|_F^2}}{ s^*_r}
  \end{align*}
  where the first inequality follows from the Young's inequality.
  To summarize, we have
  \begin{align}\label{eq:wedin-frob}
  \|\hUp^\T \U^* \|_F \|\hVp^\T \V^*\|_F \le \fr{\|\Th^* - \hTh\|_F^2}{s_r^{*2}} \;.
  \end{align}
  Theorem~\ref{thm:optspace-ours} and the inequality $\|\hTh - \Th^*\|_F \le \|\X^{-1}\|_2 \|\Z^{-1}\|_2 \|\hK - \K^*\|_F$ conclude the proof.
\end{proof}

\section{Proof of Lemma~\ref{lem:lowoful-concentration}}
This lemma is a direct consequence of~\citet[Lemma 3]{valko14spectral}.
We just need to characterize the constant $C$ therein that upper bounds $\|\th^*\|_\Lam$.
The observation that
\begin{align*}
\|\th^*\|_\Lam
\le \sqrt{\lambda\|\th_{1:k}\|_2^2 + \lambda_\perp\|\th_{k+1:p} \|_2^2}
\le \sqrt{\lambda} S + \sqrt{\lambda_\perp} S_\perp 
\end{align*}
completes the proof.

\section{Proof of Theorem~\ref{thm:lowoful-regret}}

Let $r_t$ be the instantaneous pseudo-regret at time $t$: $r_t =
\la\th^*, \a^* \ra - \la \th^*, \a_t \ra$.  
The assumptions $\|\a_t\|_2\le 1$ and $||\th^*||_2\le B$ imply that
$r_t \le 2B$.
Using the fact that the OFUL ellipsoidal confidence set contains $\th^*$ w.p. $\ge
1-\dt$, one can show that $ r_t \le 2 \sqrt{\beta_t}
\|\x_t\|_{\V_{t-1}^{-1}}$ as shown in~\citet[Theorem 3]{ay11improved}.
Then, using the monotonicity of $\beta_t$,
\begin{align*}
r_t 
&\le \min\{ 2B, 2 \sqrt{\beta_{{t}}} \|\x_t\|_{\V_{t-1}^{-1}}\}
\le \min\{ 2B, 2 \sqrt{\beta_{{T}}} \|\x_t\|_{\V_{t-1}^{-1}}\}
\\&= 2\sqrt{\beta_T} \min\{ B/\sqrt{\beta_T},\|\x_t\|_{\V_{t-1}^{-1}}\}
\\\implies \sum_{t=1}^T r_t
&\stackrel{(a)}{\le}  2\sqrt{\beta_T} \sqrt{T\sum_{t=1}^T \min\{ B^2/\beta_T, \|\x_t\|^2_{\V_{t-1}^{-1}}\}}
\\&\stackrel{(b)}{\le}  2\sqrt{\beta_T} \sqrt{T \max\{2,B^2/\beta_T\} \sum_{t=1}^T   \log\lt(1 + \|\x_t\|^2_{\V_{t-1}^{-1}}\rt)}
\\&\stackrel{(c)}{\le} 2\sqrt{\max\{2,1/\lam\}} \sqrt{\beta_T} \sqrt{\log\fr{|\V_T|}{|\Lam|}}\sqrt{T}
\end{align*}
where $(a)$ is due to Cauchy-Schwarz, $(b)$ is due to $\min\{a,x\} \le \max\{2,a\} \log(1+x), \forall a,x>0$ (see~\citet[Lemma 3]{jun17scalable}), and $(c)$ is by $\sum_{t=1}^T \log(1 + \|\x_t\|^2_{\V_{t-1}^{-1}}) = \log\fr{|\V_T|}{|\Lam|}$ (see~\citet[Lemma 11]{ay11improved}) and $\beta_T \ge (\sqrt{\lam}B + \sqrt{\lamp}B_\perp)^2 \ge \lam B^2$.

\section{Proof of Corollary~\ref{cor:lowoful-regret}}

The lemma from~\citet{valko14spectral} characterizes how large $\log\fr{|\V_T|}{|\Lam|}$ can be.
\begin{lem}\cite[Lemma 5]{valko14spectral} \label{lem:logdet}
  For any $T$, let $\Lam = \diag([\lam_1,\ldots,\lam_p])$.
  \begin{align*}
  \log \fr{|\V_T|}{|\Lam|} \le \max  \sum_{i=1}^p \log\lt( 1 + \fr{t_i}{\lam_i}\rt)
  \end{align*}
  where the maximum is taken over all possible positive real numbers $t_1,\ldots,t_p$ such that $\sum_{i=1}^p t_i = T $.
\end{lem}
We specify a desirable value of $\lamp$ in the following lemma.
\begin{lem}\label{lem:logdet-2}
    If $\lamp = \fr{T}{k\log\lt(1+\fr{T}{\lam}\rt) }$, then 
    \begin{align*}
    \log\fr{|\V_T|}{|\Lam|} \le 2 k\log\lt(1+\fr{T}{\lam}\rt)
    \end{align*}
\end{lem}
\begin{proof}
  Inheriting the setup of Lemma~\ref{lem:logdet}, 
  \begin{align*}
  \log\fr{|\V_T|}{|\Lam|} 
  &\le \max \sum_{i=1}^p \log (1 + \fr{t_i}{\lam_i})
  \\&\le k \log( 1 + \fr{T}{\lam})  + \sum_{i=k+1}^{p} \log( 1 + \fr{t_i}{\lamp})
  \end{align*}
  Then,
  \[
  \sum_{i=k+1}^{p} \log( 1 + \fr{t_i}{\lamp}) 
  \le \sum_{i=k+1}^{p} \fr{t_i}{\lamp}
  \le \fr{T}{\lamp}
  = k \log(1 + T/\lam) \;.
  \]
\end{proof}
By Lemma~\ref{lem:logdet-2}, the regret bound is, ignoring constants,
\begin{align*}
\lt(\sig k \log(1 + T/\lam) + \sqrt{k \log(1+T/\lam)} \cdot (\sqrt{\lam}S + \sqrt{\lamp}S_\perp)\rt)\cdot \sqrt{T} 
&=\tilO\lt(\lt(\sig k + \sqrt{k} \cdot (\sqrt{\lambda} S + \sqrt{\lambda_\perp} S_\perp)  \rt) \sqrt{T}\rt)
\\&= \tilO( (\sig k + \sqrt{k\lambda}S + \sqrt{k} \cdot \sqrt{\fr{T}{k}} S_\perp ) \sqrt{T})
\\&= \tilO( (\sig k + \sqrt{k\lambda}S + \sqrt{T} S_\perp ) \sqrt{T})
\end{align*}

\section{Proof of Theorem~\ref{thm:estr-regret} and Corollary~\ref{cor:estr-regret}}

Let us define $r_t = \max_{\x\in\cX, \z\in\cZ} \x^\T \Th^* \z - \x_t^\T \Th^* \z_t$, the instantaneous regret at time $t$.
Using $\max_{\|x\|_2, \|z\|_2\le1}$ $ |\x^\T\Th^*\z| \le \|\Th^*\|_2$, we bound the cumulative regret incurred up to the stage 1 as $\sum_{t=1}^{T_1} r_t \le 2 S_2 T_1$.
In the second stage, by the choices of $S$ and $S_F$ of~\eqref{eq:def-B-Bp},
\begin{align*}
\sum_{t=T_1+1}^{T_2} r_t  
&= \tilO( (\sig k + \sqrt{k\lam}S + \sqrt{T_2} S_\perp)\sqrt{T_2})
\\&= \tilO\lt( \lt(\sig k + \sqrt{k\lam}S_F + \sqrt{T_2} \cdot \|\X^{-1}\|_2^2 \|\Z^{-1}\|_2^2 C_1^2 \fr{S_2^5}{S_r^6} \sig^2 d^3 r\cdot \fr{1}{T_1} \rt)\sqrt{T_2}\rt) 
\end{align*}
Then, the overall regret is, using $T_2 \le T$
\begin{align*}
   \sum_{t=1}^T r_t = \tilO\lt(  s^*_1 T_1 + T \cdot \|\X^{-1}\|_2^2 \|\Z^{-1}\|_2^2 \fr{S_2^5}{S_r^6} \sig^2 d^3 r\cdot \fr{1}{T_1} \rt)
\end{align*}
With the choice of $T_1 = \Theta\lt( \sqrt{T\|\X^{-1}\|_2^2 \|\Z^{-1}\|_2^2 \fr{S_2^4}{S_r^6} \sig^2 d^3 r} \rt) $, the regret is
\begin{align*}
\tilO\lt( \fr{S_2^3}{S_r^3} \|\X^{-1}\|_2 \|\Z^{-1}\|_2 \sig d^{3/2}\sqrt{rT}  \rt)
\end{align*}

\section{Heuristics for selecting arms in stage 1}
\label{sec:heuristics}

We describe heuristics for solving~\eqref{eq:stage1-optim} when the arm set is finite.

Let $\X \in \RR^{N_1 \times d_1}$ be a matrix that takes each arm $\x \in \cX$ as its row.
One way to develop algorithms for solving~\eqref{eq:stage1-optim} is to relax the cardinality constraint to a continuous one:
\begin{align*}
\begin{array}{rrclcl}
\displaystyle \min_{\bflam} & \multicolumn{3}{l}{ -t }  \\
\mbox{s.t.} & \X^\T \diag(\bflam)\X &\succeq& t \I \\
& \lam_i  &\ge& 0 \quad \forall i \in [N_1]\\
& \sum_{i=1}^{N_1} \lam_i &=& 1\\
\end{array}
\end{align*}
We then choose the top $d_1$ arms with the largest $\lam_i$.

We found that choosing the best among the solution above and 20 candidate subsets drawn uniformly at random (total 21 candidate subsets) returns reasonable solutions for our purpose.

%
%
%

\end{document}